\documentclass{fairmeta}
\usepackage{amsmath,amssymb}
\usepackage{amsthm}
\usepackage{algorithm}
\usepackage{algpseudocode}
\usepackage{array}
\usepackage{enumitem}
\usepackage{tikz}
\usepackage{booktabs}
\usetikzlibrary{arrows.meta,positioning,shapes.geometric,patterns}
\usepackage{colortbl}
\definecolor{panelgray}{RGB}{242,242,242}
\definecolor{averagegray}{RGB}{248,248,248}
\definecolor{oursblue}{RGB}{232,242,252}
\newcommand{\bestours}[1]{\cellcolor{oursblue}\textbf{#1}}

\long\def\abstracttext{%
Black-box On-Policy Distillation (OPD) seeks to improve a student from its own
generations when the teacher provides sampled responses but not token
probabilities.  Adversarial distillation offers one route: it learns a
discriminator over prompt-matched teacher and student responses and uses its
score as the policy reward.  However, sampling discriminator negatives from
the latest student at each step couples the learned reward to a negative
distribution that changes after every policy update. We address this moving-target problem with persistent-negative
adversarial distillation, a live-pool method that replaces a fraction of
each discriminator batch with historical, prompt-matched teacher--student
comparisons. Under matched discriminator compute, historical comparisons train the discriminator, while GRPO remains on-policy with fresh student responses. Our analysis identifies the
Bayes-optimal reward as a teacher-to-negative log-density ratio and, under
explicit assumptions, shows how persistent negatives anchor the discriminator
and reduce reward-estimation MSE relative to fresh-negative training. Across two student families, three judges, and four judged-chat benchmarks,
persistent-negative adversarial distillation consistently improves performance over current methods at matched discriminator compute. It also yields
smoother fresh-policy discriminator trajectories, with fewer below-chance
dips. These findings identify the discriminator's negative distribution as an
important design axis in black-box on-policy distillation.}

\title{Persistent Negatives for Adversarial Black-Box On-Policy Distillation}

\author[1,*]{Haixu Ma}
\author[1,2,*,\dag]{Saad Lahrichi}
\author[1]{Weiwei Li}
\author[1]{Kevin Han}
\author[1]{Weiqiang Wu}
\author[1]{Peggy Yang}
\author[1]{Dongzhuo Li}
\author[1]{Ruiyi Li}
\author[1]{Serena Li}
\author[1]{Gedi Zhou}
\author[1]{Mingze Gao}
\author[1]{Abhishek Kumar}
\author[1]{Xiangjun Fan}
\author[1]{Lizhu Zhang}

\affiliation[1]{Meta AI}
\affiliation[2]{University of Missouri}

\contribution[*]{Equal contribution}
\contribution[\dag]{This work was performed while at Meta}

\abstract{\abstracttext}

\date{\today}
\newtheorem{proposition}{Proposition} 
\newtheorem{theorem}{Theorem}[section]

\begin{document}

\maketitle

% Main body, sections 1-10. Style-agnostic: no title, abstract, bibliography
% or appendix switch here.
% =============================================================================
\section{Introduction}

Knowledge distillation transfers capabilities from a strong teacher to a
smaller student.  In the white-box setting, the student can match the
teacher's token probabilities or hidden representations
\citep{hinton2015distilling}.  Proprietary teachers instead commonly expose
only generated text.  Sequence-level knowledge distillation
(SeqKD;~\citealp{kim2016seqkd}) is compatible with this interface, but it
trains the student only on teacher trajectories.  On-policy distillation
addresses this mismatch by learning from the student's own generations, yet
established likelihood-based objectives still require teacher probabilities
\citep{gu2024minillm,agarwal2024onpolicy}. The central challenge is therefore how to provide useful feedback on the
student's own responses when the teacher supplies only generated text, without
token-level probabilities.

Adversarial distillation addresses this gap by learning a discriminator from prompt-matched teacher and student responses and using its scores as sequence-level rewards~\citep{ye2025gad}. For the same prompt, the
discriminator learns to score a teacher response above a student response;
its scores then serve as sequence-level rewards for a GRPO update
\citep{shao2024deepseekmath}.  This converts text-only teacher outputs into an
adaptive policy-learning signal, but it also creates a second estimation
problem: the reward model must be learned from a negative distribution over
student responses.  When all negatives are freshly sampled, that distribution
changes after every policy update.  Failure modes that are rare in the latest
rollout can disappear from discriminator training even when they remain
reachable by the student.

Our key observation is that these two notions of on-policy data are distinct.
The policy update must use responses sampled from the current student, but the
discriminator need not discard every valid comparison from earlier students.
We therefore introduce \emph{persistent-negative adversarial distillation}.
Our primary live-pool method maintains a bounded collection of complete,
prompt-matched teacher--student comparisons and replaces a fixed fraction of
each discriminator batch with historical comparisons.  The updated
discriminator then scores only fresh current-student responses for GRPO, so
historical data affect the policy exclusively through the learned reward and
the student update remains on-policy. Figure~\ref{fig:gad-loop} illustrates the complete training loop.
At each iteration, the current student generates fresh responses; fresh and
historical prompt-matched comparisons train the discriminator; and the updated discriminator supplies rewards for a GRPO update
computed from fresh student rollouts. The resulting student then generates the
responses used in the next iteration. 

This design admits a theoretical account. First, the Bayes-optimal
Bradley--Terry score \citep{bradley1952rank} is the log-density ratio between
the teacher distribution and the negative distribution, so changing the
negative pool changes the learned reward rather than merely its offset.
Second, persistent comparisons provide
explicit lower bounds on the historical loss contribution and discriminator
curvature. Third, under variance-reduction and bias-control
assumptions, the pooled estimator has lower reward MSE, which tightens the
error bounds for the standardized advantages and local reward-gradient term
used by GRPO.

\begin{figure}[t]
\centering
\includegraphics[width=\linewidth]{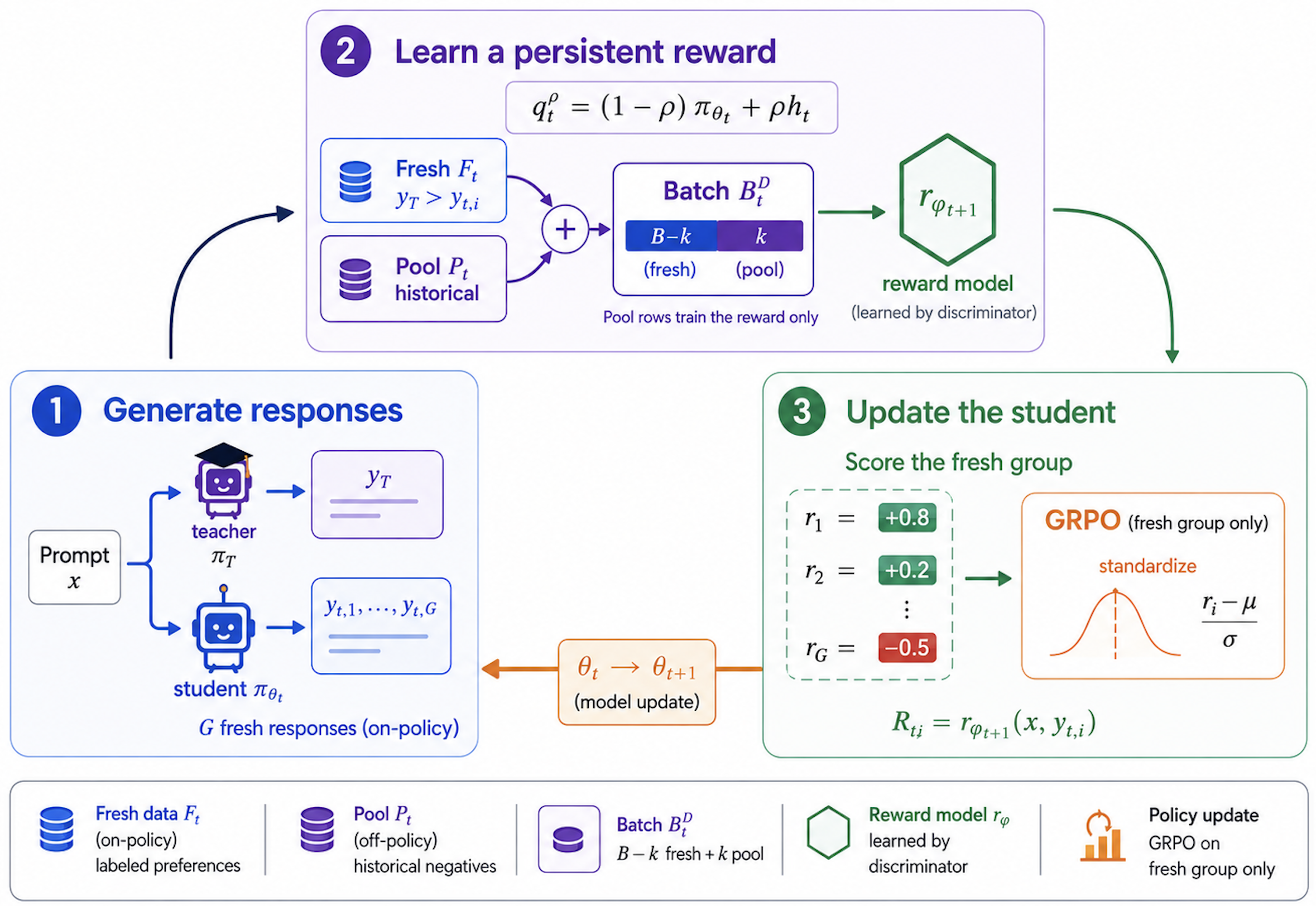}
\caption{\textbf{Workflow of persistent-negative on-policy distillation.}
At iteration $t$, each prompt $x$ is paired with a teacher response $y_T$ and
a group of fresh student responses $y_{t,1:G}$. Fresh comparisons $F_t$ are
combined with historical comparisons from $P_t$ to form the fixed-size
discriminator batch $B_t^D$, from which the updated reward
$r_{\varphi_{t+1}}$ is learned. The discriminator then scores the fresh student
group, and the resulting within-group advantages drive the GRPO update
$\theta_t\rightarrow\theta_{t+1}$ for student.}
\label{fig:gad-loop}
\end{figure}

Using GPT-5 Chat responses as black-box teacher supervision%
~\citep{openai2025gpt5chat,ye2025gad}, we evaluate
Qwen2.5-7B-Instruct~\citep{qwen2024qwen25} and
Llama-3.1-8B-Instruct~\citep{grattafiori2024llama} students across three judges and
four judged-chat benchmarks. At matched discriminator compute,
persistent-negative training outperforms current-only training%
~\citep{ye2025gad} in cross student--judge--dataset evaluation, with
equal-set-weighted gains of $+1.0$ to $+2.0$ percentage points. The resulting
students also improve over their undistilled initializations by $+1.7$ to
$+5.3$ points. In addition, persistent negatives reduce the temporal
standard deviation of fresh-policy discriminator accuracy and reduce below-chance dips from six to one.

\textit{Our contributions are}:
\begin{itemize}[leftmargin=1.3em,itemsep=2pt,topsep=3pt,parsep=0pt]
\item We identify the discriminator's negative distribution as a
first-class design choice in black-box adversarial distillation and
introduce persistent-negative training, while
keeping GRPO updates fully on-policy and discriminator compute fixed.

\item We characterize the Bayes-optimal reward as a teacher-to-negative
density ratio and establish explicit historical-loss and curvature bounds,
together with a bias--variance condition under which persistent negatives
reduce reward MSE.

\item Experiments across two student families, three judges, and both
in-domain and cross-dataset evaluations show broadly consistent gains over
current-only training, together with reduced discriminator volatility and
fewer below-chance dips.
\end{itemize}

% =============================================================================
\section{Related Work}
\label{sec:related}

\textbf{Black-box and on-policy distillation.}
Classical knowledge distillation matches teacher output distributions
\citep{hinton2015distilling}, whereas response-based methods learn from sampled
teacher text~\citep{kim2016seqkd,jiang2023lion}.  On-policy distillation trains
on student-generated trajectories, but conventional objectives require teacher
probabilities on those trajectories
\citep{gu2024minillm,agarwal2024onpolicy,lu2025opd}.  Black-box alternatives
instead derive sequence-level supervision from text: OVD uses teacher-provided
verbal scores~\citep{xiong2026ovd}, ROPD constructs prompt-specific
rubrics~\citep{fang2026ropd}, and GAD learns an adaptive discriminator from
teacher--student responses~\citep{ye2025gad}.  We develop the
discriminator-based approach by introducing persistent comparisons that
stabilize its learned reward while preserving fresh on-policy student updates. In addition, we
focus on an orthogonal question within discriminator-based distillation: which
student-response distribution should train the learned reward?

\textbf{Adversarial reward learning.}
Generative adversarial learning trains a discriminator to distinguish samples
from two distributions~\citep{goodfellow2014gan}, while adversarial imitation
learning uses the resulting discriminator as a policy reward
\citep{ho2016gail}.  Under logistic classification, the optimal discriminator
logit recovers a log-density ratio between the positive and negative
distributions.  GAD applies this principle to teacher and student responses.
Our analysis specializes it to prompt-conditional language generation and
shows that the discriminator's reward depends explicitly on the student
distribution used to construct its negatives.

\textbf{Historical and mixed-policy negatives.}
History buffers and replay have been used to stabilize adversarial image
generation and imitation learning
\citep{shrivastava2017simgan,kostrikov2019dac}.  ORPO-Distill similarly mixes
responses from different student policies when constructing preference
pairs~\citep{singh2025orpodistill}.  In ORPO-Distill, these responses directly
enter the student's preference objective.  Our historical responses instead
remain paired with their prompts and teacher responses and train only the
discriminator; GRPO continues to score and optimize fresh responses sampled
from the current student.

% % =============================================================================
\section{Persistent-Negative Adversarial Distillation}
\label{sec:method}

Figure~\ref{fig:gad-loop} provides an overview of the alternating
training procedure. We now formalize its discriminator, negative distribution,
and policy updates.
% =============================================================================

\subsection{Adversarial reward formulation}

Let \(\mathcal D\) be a prompt distribution,
\(\tau(\cdot\mid x)\) the teacher-response distribution, and
\(\pi_{\theta_t}(\cdot\mid x)\) the student policy at iteration \(t\).
We observe teacher-response samples
\(\mathcal T=\{(x,y_T)\}\), where \(x\sim\mathcal D\) and
\(y_T\sim\tau(\cdot\mid x)\), but cannot query teacher token probabilities.
We therefore learn a sequence-level score \(r_\varphi(x,y)\) from
prompt-matched teacher and student responses.

This construction connects adversarial generation to policy distillation.
In a GAN, a discriminator separates observed from generated samples while the
generator changes its distribution to become harder to distinguish
\citep{goodfellow2014gan}.  Generative Adversarial Distillation
(GAD;~\citealp{ye2025gad}) instantiates this game for conditional text:
teacher responses play the role of observed samples and the autoregressive
student plays the generator.  For any negative response distribution
\(q(\cdot\mid x)\), define the population Bradley--Terry loss
\citep{bradley1952rank}
\begin{equation}
  \mathcal L_D(r_\varphi;q)
  =
  \mathbb E_{\substack{x\sim\mathcal D,\,
                       y_T\sim\tau(\cdot\mid x),\\
                       y_S\sim q(\cdot\mid x)}}
  \left[
    \operatorname{softplus}\!\left(
      r_\varphi(x,y_S)-r_\varphi(x,y_T)
    \right)
  \right].
  \label{eq:method-bt}
\end{equation}
The discriminator minimizes this loss, increasing the score margin between
teacher and student responses.  When \(q=\pi_\theta\), the corresponding
response-level minimax game is
\begin{equation}
  \min_\varphi\max_\theta\;
  \mathcal V(\theta,\varphi)
  =
  \mathcal L_D(r_\varphi;\pi_\theta).
  \label{eq:method-minimax}
\end{equation}
The student maximization pushes probability toward responses receiving higher
discriminator scores.  Because sampled text is discrete, we use the standard
non-saturating surrogate: the discriminator score becomes a sequence-level
reward, and GRPO performs the student update instead of differentiating
through generated tokens~\citep{shao2024deepseekmath}. Concretely, the updated discriminator scores a fresh group
\(y_{t,i}\sim\pi_{\theta_t}(\cdot\mid x)\), \(i=1,\ldots,G\), producing
\(R_{t,i}=r_{\varphi_{t+1}}(x,y_{t,i})\).
GRPO forms the promptwise
advantages
\begin{equation}
  A_{t,i}
  =
  \frac{R_{t,i}-\bar R_t}{s_{R,t}+\varepsilon}, 
  \bar R_t=G^{-1}\sum_{j=1}^{G}R_{t,j}, 
  s_{R,t}=\operatorname{std}_{j}(R_{t,j}),
  \label{eq:grpo_reward}
\end{equation}
and applies the standard clipped policy objective with a KL penalty to a
frozen reference policy ~\citep{ye2025gad}.  Only fresh student responses enter this update.

\textbf{Why the reward transfers teacher behavior.}
Theorem~\ref{thm:bayes-score} shows that the population-optimal reward is
\(r_q^\star(x,y)=\log[\tau(y\mid x)/q(y\mid x)]+c(x)\). So it favors
responses characteristic of the teacher relative to the student negatives.
GRPO increases the probability of fresh responses receiving higher
within-group rewards, transferring teacher-compatible behavior without
token-level supervision.

\subsection{Persistent negatives for a moving reward target}
\label{sec:persistent-distribution}

\begin{figure}[t]
  \centering
  \includegraphics[width=\linewidth]
  {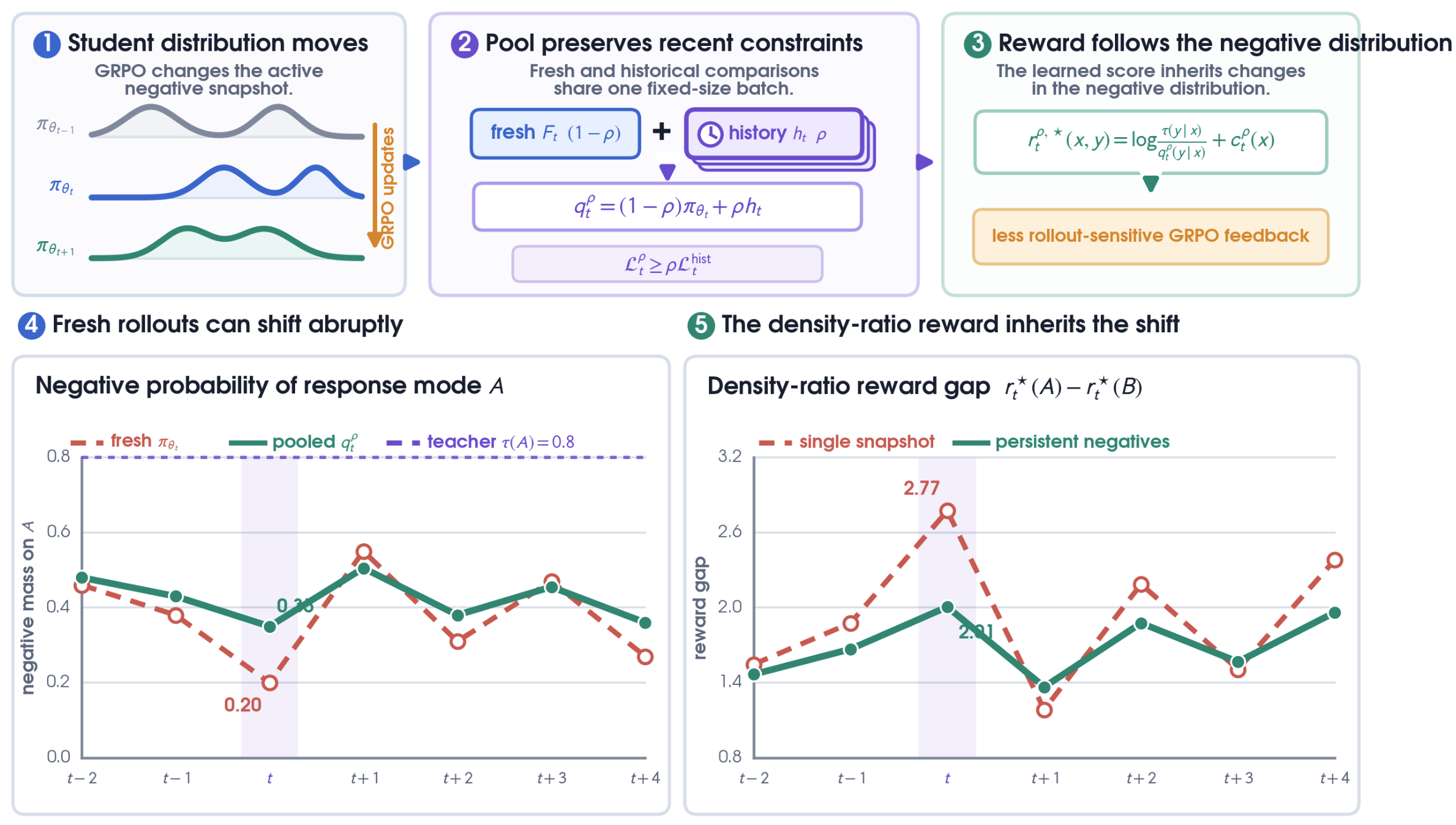}
  \caption{\textbf{Why persistent negatives help.}
  GRPO updates continually change the student distribution, so a
  discriminator trained from a single policy snapshot can lose constraints
  supplied by earlier student responses. Persistent-negative training uses
  \(q_t^\rho=(1-\rho)\pi_{\theta_t}+\rho h_t\), allowing fresh and historical
  comparisons to share a fixed-size discriminator batch. The historical term
  remains active in the discriminator objective, making the learned reward
  less dependent on any single rollout and providing steadier feedback for
  GRPO. Reward trajectories are schematic.}
  \label{fig:why-persistent-negatives}
\end{figure}

The negative distribution in Eq.~\eqref{eq:method-bt} determines the reward
being learned.  Using only \(q=\pi_{\theta_t}\) makes this target move after
every policy update: response modes absent from the latest rollout stop
contributing to discriminator training even when the student can produce them
again.  The learned reward can therefore react strongly to one policy
snapshot and forget previously observed failure modes.

We address this moving-target problem with persistent negative
distribution.  Let \(P_t\) be a bounded pool of complete comparisons
\((x,y_T,y_S)\), and \(h_t(\cdot\mid x)\) be the historical student
response distribution represented by that pool.  Once the pool is filled,
the discriminator negative distribution is
\begin{equation}
  q_t^\rho(\cdot\mid x)
  =
  (1-\rho)\pi_{\theta_t}(\cdot\mid x)
  +
  \rho h_t(\cdot\mid x),
  \qquad \rho\in[0,1).
  \label{eq:method-negative-mixture}
\end{equation}
Figure~\ref{fig:why-persistent-negatives} summarizes the mechanism.
As GRPO updates move the student distribution, the negative distribution
constructed from a single policy snapshot also moves. Mixing fresh responses
with \(h_t\) keeps recent comparisons active in the discriminator objective,
which can make the learned reward less sensitive to any individual rollout. Retaining the complete comparison preserves the prompt and teacher response
that define the pairwise target for each historical student response. The primary \emph{live pool} retains recent comparisons and refreshes them during
training. Pooled comparisons train only \(r_\varphi\); GRPO continues to use
fresh groups sampled from \(\pi_{\theta_t}\).  The special case \(\rho=0\)
gives the nonpersistent comparator.

\subsection{Alternating optimization}
\label{sec:pool-algorithm}

Let \(B\) be the discriminator batch size, \(C\) the pool capacity, and
\(F_t\) the fresh prompt-matched comparisons constructed at iteration \(t\).
We initialize \(P_0\) as empty.
% ; for a static pool, \(P_0\) is pre-collected
% and remains unchanged.  
During pool startup,
%During live-pool startup,
\(k_t=\min\{\lfloor\rho B\rfloor,|P_t|\}\) pooled rows are available, so the
realized mixture weight is \(\rho_t=k_t/B\); the remaining slots are filled
with fresh comparisons.  Algorithm~\ref{alg:persistent-negative} gives the
complete alternating update.

\begin{algorithm}[t]
\caption{Persistent-negative adversarial distillation}
\label{alg:persistent-negative}
\begingroup
\small
\setlength{\baselineskip}{0.92\baselineskip}
\begin{algorithmic}[1]
\Require Teacher-response data \(\mathcal T\), student \(\pi_{\theta_0}\), and
score model \(r_{\varphi_0}\)
\Require Pool \(P_0\), \(\rho\), \(B\), \(C\)
%, and mode \(M\in\{\mathrm{live},\mathrm{static}\}\)
\For{\(t=0,1,\ldots,N-1\)}
\State Sample prompt--teacher pairs \((x,y_T)\sim\mathcal T\)
\State Sample fresh groups
\(y_{t,i}\sim\pi_{\theta_t}(\cdot\mid x)\), \(i=1,\ldots,G\)
\State Form fresh comparisons \(F_t\gets\{(x,y_T,y_{t,i})\}\)
\State \(k_t\gets\min\{\lfloor\rho B\rfloor,|P_t|\}\)
\State \(B_t^D\gets
\Call{Sample}{F_t,B-k_t}\oplus\Call{Sample}{P_t,k_t}\)
\State \(\varphi_{t+1}\gets
\Call{Disc Update}{\varphi_t,B_t^D}\)
\State \(R_{t,i}\gets r_{\varphi_{t+1}}(x,y_{t,i})\) for fresh responses
\State Compute \(A_{t,i}\) in (\ref{eq:grpo_reward})
\State \(\theta_{t+1}\gets
\Call{GRPO Update}{\theta_t,\{y_{t,i},A_{t,i}\}}\)
%\If{\(M=\mathrm{live}\)}
\State \(P_{t+1}\gets\Call{Keep Recent}{P_t\oplus F_t,C}\)
%\Else
%  \State \(P_{t+1}\gets P_t\)
%\EndIf
\EndFor
\State \Return \(\pi_{\theta_N}\)
\end{algorithmic}
\endgroup
\end{algorithm}

Pooled comparisons replace rather than augment fresh comparisons, so every
variant uses \(B\) discriminator comparisons and the same number of
discriminator optimizer steps per iteration.  Historical data affect the
student only through the updated score model.  Moreover, because GRPO removes
a promptwise reward offset, persistence changes the policy update only when it
changes the ordering or relative spacing of scores within a fresh group.  The
following section analyzes the resulting reward target, discriminator
stability, and GRPO signal.

\section{Theoretical Analysis}
\label{sec:theory}

Our analysis provides three complementary results. The Bayes-optimal
Bradley--Terry reward is a teacher-to-negative log-density ratio. Persistent
comparisons lower-bound both the historical loss contribution and the
discriminator curvature. Finally, a bias--variance condition characterizes
when persistence reduces reward MSE. The three results formalize the mechanism illustrated in
Figure~\ref{fig:why-persistent-negatives}.

\textbf{Notation.}
We use distributions and discriminator loss defined in
Section~\ref{sec:method}. In particular, \(q_t^\rho\) is given by
Eq.~\eqref{eq:method-negative-mixture}, and \(\mathcal L_D(r;q)\) denotes
Eq.~\eqref{eq:method-bt} evaluated with score function \(r\). For an arbitrary
negative distribution \(q\), let
\(r_q^\star\in\arg\min_r\mathcal L_D(r;q)\) denote a population-optimal score
over the unrestricted class of real-valued functions. We abbreviate
\(r_t^\star:=r_{\pi_{\theta_t}}^\star\) and
\(r_t^{\rho,\star}:=r_{q_t^\rho}^\star\). Because the loss depends only on
score differences, these optima are defined up to a prompt-dependent additive
constant.

\subsection{The Bayes-optimal score is a density ratio}

\begin{theorem}[Bayes-optimal discriminator score]
\label{thm:bayes-score}
Fix a prompt \(x\). Suppose the response space is finite and
\(\tau(y\mid x),q(y\mid x)>0\) for every response \(y\). Over the unrestricted
score class,
\begin{equation}
  r_q^\star(x,y)
  =
  \log\frac{\tau(y\mid x)}{q(y\mid x)}
  +c_q(x),
  \label{eq:bayes-density-ratio}
\end{equation}
where \(c_q(x)\) is independent of \(y\). Consequently, we have
\begin{align}
  &r_t^\star(x,y)
  =
  \log\frac{\tau(y\mid x)}
            {\pi_{\theta_t}(y\mid x)}
  +c_t(x), \text{\ \ and\ \ } 
  r_t^{\rho,\star}(x,y)
  =
  \log\frac{\tau(y\mid x)}
            {q_t^\rho(y\mid x)}
  +c_t^\rho(x).
\end{align}

\end{theorem}

% \paragraph{Proof sketch.}
% For two responses \(y,z\), group the two Bradley--Terry terms containing their
% score difference \(d=r(y)-r(z)\).  The first-order condition gives
% \(d=\log[\tau(y)q(z)/(\tau(z)q(y))]
% =\log[\tau(y)/q(y)]-\log[\tau(z)/q(z)]\).
% These pairwise optima are simultaneously realized by
% Eq.~\eqref{eq:bayes-density-ratio}.  The complete proof is in
% Appendix~\ref{app:theory-proofs}.

\textbf{Example.}
Consider two response modes, \(A\) and \(B\), with
\(\tau(A,B)=(0.8,0.2)\) and
\(\pi_{\theta_t}(A,B)=(0.2,0.8)\).  Under fresh-negative training,
\(
  r_t^\star(A)-r_t^\star(B)
  =
  \log\frac{0.8}{0.2}
  -
  \log\frac{0.2}{0.8}
  =
  \log 16
  \approx2.77.
\)
If \(h_t(A,B)=(0.5,0.5)\) and \(\rho=0.5\), then
\(q_t^\rho(A)=0.5(0.2)+0.5(0.5)=0.35\) and
\(q_t^\rho(B)=0.5(0.8)+0.5(0.5)=0.65\).  The corresponding gap is
\(\log[(0.8/0.35)/(0.2/0.65)]=\log7.43\approx2.01\).
Here the current rollout may make the correct mode look unusually rare; the pool
remembers that recent students produced it half of the time and avoids turning
this transient shift into an excessively large reward gap.  It preserves the
useful ordering \(A>B\) while making the reward less reactive to one policy
snapshot.  With realistic multi-mode responses, pooling changes different
gaps non-uniformly and therefore can also change the standardized GRPO advantages.

\subsection{Persistent negatives anchor the discriminator}
  \label{sec:theory-anchor}

  For a fixed prompt, let scores lie in the centered subspace
  $\mathcal S=\{r:\mathbf 1^\top r=0\}$. Define
  $\mathcal L_t^{\mathrm{fresh}}(r)=\mathcal L_D(r;\pi_{\theta_t})$ and
  $\mathcal L_t^{\mathrm{hist}}(r)=\mathcal L_D(r;h_t)$. At the population
  level, the negative-mixture distribution gives
  $\mathcal L_t^\rho=(1-\rho)\mathcal L_t^{\mathrm{fresh}}
  +\rho\mathcal L_t^{\mathrm{hist}}$.

  \begin{proposition}[Persistent-negative lower bound]
  \label{thm:pool-lower-bound}
  For every score vector $r$ and $\rho\in[0,1)$, we have
$
    \mathcal L_t^\rho(r)
    \ge \rho\mathcal L_t^{\mathrm{hist}}(r).
$
  Moreover, suppose that, on a bounded convex subset of $\mathcal S$,
  $\nabla^2\mathcal L_t^{\mathrm{fresh}}(r)
  \succeq\mu_{\mathrm{fresh}}I_{\mathcal S}$ and
  $\nabla^2\mathcal L_t^{\mathrm{hist}}(r)
  \succeq\mu_{\mathrm{hist}}I_{\mathcal S}$, where
  $\mu_{\mathrm{fresh}},\mu_{\mathrm{hist}}\ge0$. Then
$
    \nabla^2\mathcal L_t^\rho(r)
    \succeq \rho\mu_{\mathrm{hist}}I_{\mathcal S}.
$
  \end{proposition}

  % \paragraph{Proof.}
  % The Bradley--Terry loss is nonnegative, so the mixture decomposition gives
  % Eq.~\eqref{eq:pool-loss-lower-bound}. The Hessian follows the same mixture
  % decomposition, which proves Eq.~\eqref{eq:pool-curvature-lower-bound}.
  % \hfill$\square$

  \textbf{Example.}
  Consider a prompt with three response modes: a correct answer $A$, a
  hallucinated answer $B$, and an unhelpful refusal $C$. After a policy update,
  fresh negatives may shift abruptly from $B$ to $C$. Fresh-only training then
  provides little signal about $B$, even though the student may produce it
  again. A live pool retains earlier teacher--$B$ comparisons. With
  $\rho=0.5$, the pooled objective satisfies
  $\mathcal L_t^\rho(r)\ge0.5\mathcal L_t^{\mathrm{hist}}(r)$, so achieving low
  pooled loss requires continued control of the historical comparisons
  involving $B$. If those comparisons provide curvature in the corresponding
  score direction, they contribute at least
  $0.5\mu_{\mathrm{hist}}$ to the pooled curvature.

  % =============================================================================
  \subsection{Bias--variance tradeoff of persistent negatives}
  \label{sec:theory-bias-variance}

For a reward function $f$, define the promptwise-centered seminorm as $
    \|f\|_{t,c}^2
    :=
    \mathbb E_{x\sim\mathcal D,\,y\sim\pi_{\theta_t}(\cdot\mid x)}
    \left[
      \left(
        f(x,y)
        -
        \mathbb E_{y'\sim\pi_{\theta_t}(\cdot\mid x)}f(x,y')
      \right)^2
    \right].$ This removes prompt-dependent score offsets, which do not affect pairwise
  comparisons or promptwise GRPO advantages.

  \begin{proposition}[Reward MSE under controlled pool bias]
  \label{thm:pool-bias-variance}
  Let $\widehat r_t^{\mathrm{fresh}}$ and
  $\widehat r_t^{\mathrm{pool}}$ be learned reward functions, and let
  $r_t^\star$ be the ideal fresh-negative reward from
  Theorem~\ref{thm:bayes-score}. Define
  $V_a=\mathbb E\|\widehat r_t^a-\mathbb E\widehat r_t^a\|_{t,c}^2$ for
  $a\in\{\mathrm{fresh},\mathrm{pool}\}$ and
  $B_{\mathrm{pool}}^2
  =\|\mathbb E\widehat r_t^{\mathrm{pool}}-r_t^\star\|_{t,c}^2$. Assume that $\mathbb E\widehat r_t^{\mathrm{fresh}}=r_t^\star$ and that, for
  some $\Delta_t>0$,
  $V_{\mathrm{pool}}\le V_{\mathrm{fresh}}-\Delta_t$ and
  $B_{\mathrm{pool}}^2<\Delta_t$. Then we have,
$
    \operatorname{MSE}(\widehat r_t^{\mathrm{fresh}},r_t^\star)
    =V_{\mathrm{fresh}},
    \qquad
    \operatorname{MSE}(\widehat r_t^{\mathrm{pool}},r_t^\star)
    =V_{\mathrm{pool}}+B_{\mathrm{pool}}^2.
$
  Consequently,
$
    \operatorname{MSE}(\widehat r_t^{\mathrm{fresh}},r_t^\star)
    -
    \operatorname{MSE}(\widehat r_t^{\mathrm{pool}},r_t^\star)
    \ge \Delta_t-B_{\mathrm{pool}}^2>0.
$
  \end{proposition}

  % \paragraph{Proof.}
  % The bias--variance decomposition gives the two MSE expressions in
  % Eq.~\eqref{eq:reward-mse-decomposition}. Substituting
  % $V_{\mathrm{pool}}\le V_{\mathrm{fresh}}-\Delta_t$ gives
  % Eq.~\eqref{eq:pool-mse-advantage}.
  % \hfill$\square$

Appendix~\ref{app:theory-proofs} provides complete proofs of
above theorems and propositions. 
% =============================================================================

\section{Experimental setup}
\label{sec:setup}
% =============================================================================

\subsection{Models, data, and training}

We evaluate two comparably sized student models from different families:
Qwen2.5-7B-Instruct~\citep{qwen2024qwen25} and
Llama-3.1-8B-Instruct~\citep{grattafiori2024llama}, so that we can test whether the effect of
persistent negatives transfers across model families rather than depending on
a particular architecture or initialization. For each family, we evaluate the
initial instruction-tuned checkpoint as the undistilled student.

The teacher-response corpus contains $192{,}014$ responses generated by
GPT-5 Chat~\citep{openai2025gpt5chat} and released with GAD over
LMSYS-Chat-1M~\citep{ye2025gad,zheng2024lmsyschat}. We use only sampled
teacher responses: we neither query the teacher during training nor access its token probabilities. Training begins with one epoch of supervised fine-tuning on the teacher responses, during which the discriminator is also initialized. This warmup is followed by two adversarial epochs.

Each adversarial step samples eight student responses per prompt. Our
persistent-negative configuration uses mixture weight $\rho=0.5$, uniform
pool sampling, and a bounded live pool containing $4,096$ complete comparisons
$(x,y_T,y_S)$. Once sufficient history is available, half of each
discriminator batch is sampled from the pool. Pooled comparisons replace
current comparisons within the fixed discriminator batch size $B=256$.
Consequently, GAD and persistent-negative training use the same number of
discriminator examples and optimizer steps. Complete optimization,
sequence-length, and hardware configurations appear in
Appendix~\ref{app:config}.

For each student family, we train GAD and our method with three paired random
seeds. Within each pair, both methods share the warmup checkpoint, data order,
initialization, and optimization settings. Across seeds, we vary data
shuffling, response sampling, and optimizer randomness. Both student families use the same, one-epoch supervised
warmup, two adversarial epochs, group size $G=8$, discriminator batch size
$B=256$, mixture weight $\rho=0.5$, and pool capacity of 4096 comparisons.
We change only model-specific tokenization and chat templates.
Memory-related microbatching may differ, but the effective batch size is held
fixed.

\subsection{Compared methods}

For each student model, we compare the undistilled checkpoint, SeqKD~\citep{kim2016seqkd}, standard
GAD~\citep{ye2025gad}, and our persistent-negative method. SeqKD performs supervised
fine-tuning on teacher responses without a discriminator or GRPO. GAD and
persistent-negative training resume from the same jointly initialized warmup
checkpoint. Standard GAD uses only current-policy negatives, corresponding to
$\rho=0$. Our method instead replaces half of each discriminator batch with
historical comparisons sampled from the live pool. 
Appendix~\ref{app:config} provides the complete configurations.

\subsection{Evaluation and inference}

We generate greedy responses for LMSYS ($n=479$ prompts), Dolly
($n=500$; \citealp{conover2023dolly}), Vicuna/MT-Bench
($n=80$; \citealp{zheng2023judging}), and Self-Instruct
($n=252$; \citealp{wang2023selfinstruct}). Evaluation prompts
are excluded from the distillation cohort.   

All distillation prompts are drawn from LMSYS-Chat-1M. The held-out LMSYS
evaluation set therefore measures in-domain generalization. Dolly,
Vicuna/MT-Bench, and Self-Instruct are not used during training and provide
cross-dataset out-of-domain evaluations.

An offline judge assigns scores from 1 to 10 to the candidate and reference
answers. Following \citet{ye2025gad}, we report
candidate/(candidate+reference), for which $0.5$ denotes parity. We use three
judges to reduce dependence on the preferences and calibration of any single
evaluator. Qwen2.5-72B-Instruct~\citep{qwen2024qwen25} provides a large,
established Qwen-family judge, while
Qwen3.5-27B~\citep{qwen2026qwen35} tests consistency under a newer Qwen
generation. Gemma-3-27B-IT~\citep{gemmateam2025gemma3} provides a
cross-family judge, allowing us to test whether the conclusions persist
outside the Qwen model family.

All three judges score the same fixed student responses using the same
evaluation rubric and deterministic decoding. Because their absolute scoring
scales need not be calibrated, we report results separately for each judge
and compare methods only within judge. The offline judges are independent of
the discriminator used during training.

We also measure pre-update discriminator accuracy on fresh current-policy
comparisons. Every arm uses the same prompt--teacher slice, but student
responses are generated by the corresponding current policy. This diagnostic
therefore describes the coupled student--discriminator process rather than
isolating discriminator variation alone.

% =============================================================================
\section{Results}
\label{sec:results}
% =============================================================================

\subsection{Persistent negatives improve judged-chat performance}
\label{sec:chat}

Persistent-negative training improves judged-chat performance consistently
across student and judge families. With Qwen2.5-7B-Instruct as the student,
our method outperforms GAD on every evaluation set under all three judges
(Table~\ref{tab:qwen-chat}). The equal-set-weighted gains over GAD are
$+1.1\%$ under Qwen2.5-72B-Instruct, $+1.3\%$ under Qwen3.5-27B, and
$+1.9\%$ under Gemma-3-27B-IT. Relative to the undistilled student, the
corresponding gains are $+1.7\%$, $+2.4\%$, and $+4.6\%$.

\begin{table*}[t]
  \centering
  \footnotesize
  \caption{\textbf{Judged-chat performance with Qwen2.5-7B-Instruct.}
  Scores are percentages, with $50$ denoting parity with the reference.
  Base is the undistilled student, and GAD uses current-only negatives.
  Improvements are reported in percentage points (pp). The average weights the
  four evaluation sets equally.}
  \label{tab:qwen-chat}
  \setlength{\tabcolsep}{5.5pt}
  \renewcommand{\arraystretch}{0.98}

  \begin{tabular}{@{}lcccccc@{}}
  \toprule
  \textbf{Dataset}
  & \textbf{Base}
  & \textbf{SeqKD}
  & \textbf{GAD}
  & \cellcolor{oursblue}\shortstack{\textbf{Persistent-negative}\\
                                    \textbf{(ours)}}
  & \shortstack{\textbf{$\Delta$ vs.\ Base $\uparrow$}\\\textbf{(pp)}}
  & \shortstack{\textbf{$\Delta$ vs.\ GAD $\uparrow$}\\\textbf{(pp)}} \\
  \midrule

  \rowcolor{panelgray}
  \multicolumn{7}{@{}l}{\textbf{Qwen2.5-72B-Instruct judge}} \\
  LMSYS
  & 48.5 & 48.9 & 49.1 & \bestours{50.7} & $+2.2$ & $+1.6$ \\
  Dolly
  & 47.5 & 46.8 & 47.9 & \bestours{49.1} & $+1.6$ & $+1.2$ \\
  Vicuna
  & 48.4 & 48.6 & 49.2 & \bestours{50.2} & $+1.8$ & $+1.0$ \\
  Self-Instruct
  & 48.1 & 48.5 & 48.7 & \bestours{49.4} & $+1.3$ & $+0.7$ \\
  \rowcolor{averagegray}
  \textbf{Average}
  & 48.1 & 48.2 & 48.7 & \bestours{49.9} & $+1.7$ & $+1.1$ \\

  \midrule
  \rowcolor{panelgray}
  \multicolumn{7}{@{}l}{\textbf{Qwen3.5-27B judge}} \\
  LMSYS
  & 39.1 & 39.5 & 39.6 & \bestours{41.6} & $+2.5$ & $+2.0$ \\
  Dolly
  & 42.4 & 42.7 & 42.5 & \bestours{43.4} & $+1.0$ & $+0.9$ \\
  Vicuna
  & 42.0 & 44.7 & 44.3 & \bestours{45.0} & $+3.0$ & $+0.7$ \\
  Self-Instruct
  & 42.4 & 43.2 & 43.9 & \bestours{45.6} & $+3.2$ & $+1.7$ \\
  \rowcolor{averagegray}
  \textbf{Average}
  & 41.5 & 42.5 & 42.6 & \bestours{43.9} & $+2.4$ & $+1.3$ \\

  \midrule
  \rowcolor{panelgray}
  \multicolumn{7}{@{}l}{\textbf{Gemma-3-27B-IT judge}} \\
  LMSYS
  & 41.2 & 42.8 & 44.6 & \bestours{47.2} & $+6.0$ & $+2.6$ \\
  Dolly
  & 43.2 & 43.8 & 45.1 & \bestours{46.8} & $+3.6$ & $+1.7$ \\
  Vicuna
  & 42.7 & 43.6 & 45.9 & \bestours{47.5} & $+4.8$ & $+1.6$ \\
  Self-Instruct
  & 42.4 & 42.7 & 44.6 & \bestours{46.3} & $+3.9$ & $+1.7$ \\
  \rowcolor{averagegray}
  \textbf{Average}
  & 42.4 & 43.2 & 45.1 & \bestours{47.0} & $+4.6$ & $+1.9$ \\

  \bottomrule
  \end{tabular}
  \end{table*}

The improvement is not limited to the LMSYS evaluation set, which shares the
source domain of the training corpus. Persistent-negative training also
outperforms GAD on Dolly, Vicuna, and Self-Instruct, providing evidence of
cross-dataset generalization. Although the judges use different absolute
score scales, they agree on the direction of every comparison.

  \begin{table*}[t]
  \centering
  \footnotesize
  \caption{\textbf{Judged-chat performance with Llama-3.1-8B-Instruct.}}
  \label{tab:llama-chat}
  \setlength{\tabcolsep}{5.5pt}
  \renewcommand{\arraystretch}{0.98}

  \begin{tabular}{@{}lcccccc@{}}
  \toprule
  \textbf{Dataset}
  & \textbf{Base}
  & \textbf{SeqKD}
  & \textbf{GAD}
  & \cellcolor{oursblue}\shortstack{\textbf{Persistent-negative}\\
                                    \textbf{(ours)}}
  & \shortstack{\textbf{$\Delta$ vs.\ Base $\uparrow$}\\\textbf{(pp)}}
  & \shortstack{\textbf{$\Delta$ vs.\ GAD $\uparrow$}\\\textbf{(pp)}} \\
  \midrule

  \rowcolor{panelgray}
  \multicolumn{7}{@{}l}{\textbf{Qwen2.5-72B-Instruct judge}} \\
  LMSYS
  & 46.7 & 47.2 & 48.5 & \bestours{49.4} & $+2.7$ & $+0.9$ \\
  Dolly
  & 45.8 & 45.9 & 46.7 & \bestours{48.2} & $+2.4$ & $+1.5$ \\
  Vicuna
  & 45.2 & 47.7 & 48.8 & \bestours{49.8} & $+4.6$ & $+1.0$ \\
  Self-Instruct
  & 46.3 & 47.2 & 48.2 & \bestours{48.9} & $+2.6$ & $+0.7$ \\
  \rowcolor{averagegray}
  \textbf{Average}
  & 46.0 & 47.0 & 48.1 & \bestours{49.1} & $+3.1$ & $+1.0$ \\

  \midrule
  \rowcolor{panelgray}
  \multicolumn{7}{@{}l}{\textbf{Qwen3.5-27B judge}} \\
  LMSYS
  & 36.5 & 37.9 & 39.6 & \bestours{40.7} & $+4.2$ & $+1.1$ \\
  Dolly
  & 38.1 & 39.2 & 41.2 & \bestours{42.2} & $+4.1$ & $+1.0$ \\
  Vicuna
  & 37.6 & 40.0 & 42.8 & \bestours{44.1} & $+6.5$ & $+1.3$ \\
  Self-Instruct
  & 37.8 & 39.6 & 42.7 & \bestours{43.8} & $+6.0$ & $+1.1$ \\
  \rowcolor{averagegray}
  \textbf{Average}
  & 37.5 & 39.2 & 41.6 & \bestours{42.7} & $+5.2$ & $+1.1$ \\

  \multicolumn{7}{@{}l}{\textbf{Gemma-3-27B-IT judge}} \\
  LMSYS
  & 39.5 & 39.8 & 42.3 & \bestours{45.6} & $+6.1$ & $+3.3$ \\
  Dolly
  & 40.1 & 42.0 & 43.9 & \bestours{44.9} & $+4.8$ & $+1.0$ \\
  Vicuna
  & 39.8 & 41.9 & 43.7 & \bestours{45.8} & $+6.0$ & $+2.1$ \\
  Self-Instruct
  & 39.9 & 42.0 & 42.5 & \bestours{44.2} & $+4.3$ & $+1.7$ \\
  \rowcolor{averagegray}
  \textbf{Average}
  & 39.8 & 41.4 & 43.1 & \bestours{45.1} & $+5.3$ & $+2.0$ \\

  \bottomrule
  \end{tabular}
  \end{table*}

The same pattern transfers to Llama-3.1-8B-Instruct
(Table~\ref{tab:llama-chat}). Persistent-negative training improves the
equal-set-weighted average over GAD by $+1.0\%$ under Qwen2.5-72B-Instruct,
$+1.1\%$ under Qwen3.5-27B, and $+2.0\%$ under Gemma-3-27B-IT. Relative to
the undistilled Llama student, the corresponding gains are $+3.1$, $+5.2\%$,
and $+5.3\%$. The Qwen student retains higher absolute scores, but the
persistent-negative improvement over GAD is similar across the two model
families.

Across both students, all three judges favor persistent-negative training
over GAD on all four evaluation sets, yielding positive point estimates in
all 24 student--judge--dataset comparisons. These comparisons are correlated
because they share prompts, references, and evaluation rubrics; we therefore
treat their directional consistency as robustness evidence rather than as 24
independent tests. 
% In our smaller-cohort analysis, persistent-negative training also exceeds GAD
% in the point-estimate average under both Qwen judges. The positive contrast is
% not confined to prompts for which the persistent-negative checkpoint produces
% longer responses.

\subsection{Persistent negatives produce smoother discriminator trajectories}
\label{sec:stability}

We measure pre-update discriminator accuracy on fresh current-policy
comparisons. In the seed run, persistent-negative training reduces
the temporal standard deviation from $0.087$ to $0.068$, raises the minimum
accuracy from $0.414$ to $0.461$, and reduces below-chance dips from six to
one. Because each student generates its own negatives, this diagnostic
describes the coupled student--discriminator trajectory rather than the
discriminator in isolation. Appendix~\ref{app:discriminator-diagnostics}
reports the complete statistics and common-probe analysis.

% =============================================================================
\section{Conclusion}
% =============================================================================

We introduce persistent-negative adversarial distillation for transferring
  behavior from a black-box teacher without access to token probabilities. A
  sequence-level discriminator learns rewards from prompt-matched teacher and
  student responses, while GRPO updates the student using fresh on-policy
  generations. Standard adversarial distillation trains the discriminator only
  against the current student distribution, so previously observed failure modes
  can disappear from discriminator training as the policy changes. We address
  this moving-target problem by retaining historical student responses as
  persistent negatives. This design preserves broader discriminator supervision
  while maintaining a fully on-policy student update.

\clearpage
\bibliographystyle{assets/plainnat}
\bibliography{refs}

@article{ye2025gad,
  title         = {Black-Box On-Policy Distillation of Large Language Models},
  author        = {Ye, Tianzhu and Dong, Li and Chi, Zewen and Wu, Xun and
                   Huang, Shaohan and Wei, Furu},
  journal       = {arXiv preprint arXiv:2511.10643},
  year          = {2025},
  eprint        = {2511.10643},
  archivePrefix = {arXiv}
}

@article{hinton2015distilling,
  title         = {Distilling the Knowledge in a Neural Network},
  author        = {Hinton, Geoffrey and Vinyals, Oriol and Dean, Jeff},
  journal       = {arXiv preprint arXiv:1503.02531},
  year          = {2015},
  eprint        = {1503.02531},
  archivePrefix = {arXiv}
}

@inproceedings{kim2016seqkd,
  title     = {Sequence-Level Knowledge Distillation},
  author    = {Kim, Yoon and Rush, Alexander M.},
  booktitle = {Proceedings of the 2016 Conference on Empirical Methods in
               Natural Language Processing},
  pages     = {1317--1327},
  year      = {2016}
}

@inproceedings{gu2024minillm,
  title     = {{MiniLLM}: Knowledge Distillation of Large Language Models},
  author    = {Gu, Yuxian and Dong, Li and Wei, Furu and Huang, Minlie},
  booktitle = {International Conference on Learning Representations},
  year      = {2024}
}

@inproceedings{agarwal2024onpolicy,
  title     = {On-Policy Distillation of Language Models: Learning from
               Self-Generated Mistakes},
  author    = {Agarwal, Rishabh and Vieillard, Nino and Zhou, Yongchao and
               Stanczyk, Piotr and Ramos Garea, Sabela and Geist, Matthieu and
               Bachem, Olivier},
  booktitle = {The Twelfth International Conference on Learning Representations},
  year      = {2024}
}

@misc{lu2025opd,
  title        = {On-Policy Distillation},
  author       = {Lu, Kevin and {Thinking Machines Lab}},
  year         = {2025},
  howpublished = {Thinking Machines Lab: Connectionism},
  doi          = {10.64434/tml.20251026},
  url          = {https://thinkingmachines.ai/blog/on-policy-distillation}
}

@article{bradley1952rank,
  title   = {Rank Analysis of Incomplete Block Designs: I. The Method of
             Paired Comparisons},
  author  = {Bradley, Ralph Allan and Terry, Milton E.},
  journal = {Biometrika},
  volume  = {39},
  number  = {3/4},
  pages   = {324--345},
  year    = {1952}
}

@article{shao2024deepseekmath,
  title         = {{DeepSeekMath}: Pushing the Limits of Mathematical Reasoning
                   in Open Language Models},
  author        = {Shao, Zhihong and Wang, Peiyi and Zhu, Qihao and Xu, Runxin
                   and Song, Junxiao and Bi, Xiao and Zhang, Haowei and Zhang,
                   Mingchuan and Li, Y. K. and Wu, Y. and Guo, Daya},
  journal       = {arXiv preprint arXiv:2402.03300},
  year          = {2024},
  eprint        = {2402.03300},
  archivePrefix = {arXiv}
}

@inproceedings{ho2016gail,
  author    = {Ho, Jonathan and Ermon, Stefano},
  title     = {Generative Adversarial Imitation Learning},
  booktitle = {Advances in Neural Information Processing Systems},
  volume    = {29},
  year      = {2016}
}

@inproceedings{kostrikov2019dac,
  author    = {Kostrikov, Ilya and Agrawal, Kumar Krishna and Dwibedi,
               Debidatta and Levine, Sergey and Tompson, Jonathan},
  title     = {Discriminator-Actor-Critic: Addressing Sample Inefficiency and
               Reward Bias in Adversarial Imitation Learning},
  booktitle = {International Conference on Learning Representations},
  year      = {2019}
}

@inproceedings{goodfellow2014gan,
  title     = {Generative Adversarial Nets},
  author    = {Goodfellow, Ian J. and Pouget-Abadie, Jean and Mirza, Mehdi and
               Xu, Bing and Warde-Farley, David and Ozair, Sherjil and
               Courville, Aaron and Bengio, Yoshua},
  booktitle = {Advances in Neural Information Processing Systems},
  volume    = {27},
  year      = {2014}
}

@inproceedings{shrivastava2017simgan,
  title     = {Learning from Simulated and Unsupervised Images through
               Adversarial Training},
  author    = {Shrivastava, Ashish and Pfister, Tomas and Tuzel, Oncel and
               Susskind, Josh and Wang, Wenda and Webb, Russ},
  booktitle = {IEEE Conference on Computer Vision and Pattern Recognition},
  pages     = {2107--2116},
  year      = {2017}
}

@article{efraimidis2006weighted,
  title   = {Weighted Random Sampling with a Reservoir},
  author  = {Efraimidis, Pavlos S. and Spirakis, Paul G.},
  journal = {Information Processing Letters},
  volume  = {97},
  number  = {5},
  pages   = {181--185},
  year    = {2006}
}

@article{qwen2024qwen25,
  title         = {{Qwen2.5} Technical Report},
  author        = {{Qwen Team}},
  journal       = {arXiv preprint arXiv:2412.15115},
  year          = {2024},
  eprint        = {2412.15115},
  archivePrefix = {arXiv}
}

@inproceedings{zheng2024lmsyschat,
  title     = {{LMSYS-Chat-1M}: A Large-Scale Real-World {LLM} Conversation
               Dataset},
  author    = {Zheng, Lianmin and Chiang, Wei-Lin and Sheng, Ying and Li,
               Tianle and Zhuang, Siyuan and Wu, Zhanghao and Zhuang, Yonghao
               and Li, Zhuohan and Lin, Zi and Xing, Eric P. and Gonzalez,
               Joseph E. and Stoica, Ion and Zhang, Hao},
  booktitle = {International Conference on Learning Representations},
  year      = {2024}
}

@inproceedings{zheng2023judging,
  title     = {Judging {LLM}-as-a-Judge with {MT-Bench} and Chatbot Arena},
  author    = {Zheng, Lianmin and Chiang, Wei-Lin and Sheng, Ying and Zhuang,
               Siyuan and Wu, Zhanghao and Zhuang, Yonghao and Lin, Zi and Li,
               Zhuohan and Li, Dacheng and Xing, Eric P. and Zhang, Hao and
               Gonzalez, Joseph E. and Stoica, Ion},
  booktitle = {Advances in Neural Information Processing Systems},
  volume    = {36},
  year      = {2023}
}

@inproceedings{wang2023selfinstruct,
  title     = {Self-Instruct: Aligning Language Models with Self-Generated
               Instructions},
  author    = {Wang, Yizhong and Kordi, Yeganeh and Mishra, Swaroop and Liu,
               Alisa and Smith, Noah A. and Khashabi, Daniel and Hajishirzi,
               Hannaneh},
  booktitle = {Proceedings of the 61st Annual Meeting of the Association for
               Computational Linguistics},
  pages     = {13484--13508},
  year      = {2023}
}

@misc{conover2023dolly,
  title        = {Free Dolly: Introducing the World's First Truly Open
                  Instruction-Tuned {LLM}},
  author       = {Conover, Mike and Hayes, Matt and Mathur, Ankit and Xie,
                  Jianwei and Wan, Jun and Shah, Sam and Ghodsi, Ali and
                  Wendell, Patrick and Zaharia, Matei and Xin, Reynold},
  howpublished = {Databricks blog},
  year         = {2023},
  url          = {https://www.databricks.com/blog/2023/04/12/dolly-first-open-commercially-viable-instruction-tuned-llm}
}

@inproceedings{jiang2023lion,
  title={Lion: Adversarial distillation of proprietary large language models},
  author={Jiang, Yuxin and Chan, Chunkit and Chen, Mingyang and Wang, Wei},
  booktitle={Proceedings of the 2023 Conference on Empirical Methods in Natural Language Processing},
  pages={3134--3154},
  year={2023}
}

@article{xiong2026ovd,
title   = {{OVD}: On-Policy Verbal Distillation},
author  = {Xiong, Jing and Shen, Hui and Gong, Shansan and Cheng, Yuxin
and Shen, Jianghan and Tao, Chaofan and Tan, Haochen and
Bai, Haoli and Shang, Lifeng and Wong, Ngai},
journal = {arXiv preprint arXiv:2601.21968},
year    = {2026},
url     = {https://arxiv.org/abs/2601.21968}
}

@article{fang2026ropd,
title   = {Rubric-Based On-Policy Distillation},
author  = {Fang, Junfeng and Hong, Zhepei and Zheng, Mao and Song, Mingyang
and Li, Gengsheng and Jiang, Houcheng and Zhang, Dan and
Guo, Haiyun and Wang, Xiang and Chua, Tat-Seng},
journal = {arXiv preprint arXiv:2605.07396},
year    = {2026},
url     = {https://arxiv.org/abs/2605.07396}
}

@article{singh2025orpodistill,
title   = {{ORPO-Distill}: Mixed-Policy Preference Optimization for
Cross-Architecture {LLM} Distillation},
author  = {Singh, Aasheesh and Vaddina, Vishal and Birru, Dagnachew},
journal = {arXiv preprint arXiv:2509.25100},
year    = {2025},
note    = {Accepted at the NeurIPS 2025 Workshop on Efficient Reasoning},
url     = {https://arxiv.org/abs/2509.25100}
}

@misc{openai2025gpt5chat,
author       = {{OpenAI}},
title        = {{GPT-5 Chat}},
year         = {2025},
howpublished = {\url{https://developers.openai.com/api/docs/models/gpt-5-chat-latest}},
note         = {Model documentation; accessed September 12, 2026}
}

@article{grattafiori2024llama,
  title={The llama 3 herd of models},
  author={Grattafiori, Aaron and Dubey, Abhimanyu and Jauhri, Abhinav and Pandey, Abhinav and Kadian, Abhishek and Al-Dahle, Ahmad and Letman, Aiesha and Mathur, Akhil and Schelten, Alan and Vaughan, Alex and others},
  journal={arXiv preprint arXiv:2407.21783},
  year={2024}
}

@article{gemmateam2025gemma3,
    title   = {Gemma 3 Technical Report},
    author  = {{Gemma Team} and others},
    journal = {arXiv preprint arXiv:2503.19786},
    year    = {2025},
    url     = {https://arxiv.org/abs/2503.19786}
  }

@misc{qwen2026qwen35,
    title  = {{Qwen3.5}: Towards Native Multimodal Agents},
    author = {{Qwen Team}},
    year   = {2026},
    month  = {February},
    url    = {https://qwen.ai/blog?id=qwen3.5}
  }

\clearpage
\beginappendix
% Appendices shared by both style variants.

% =============================================================================
\section{Implementation and training details}
\label{app:config}
% =============================================================================

\begin{table}[h]
\centering\small
\caption{Training configuration. Pool rows describe the only intervention.}
\label{tab:config}
\begin{tabular}{@{}p{.19\linewidth}p{.76\linewidth}@{}}
\toprule
Component & Configuration \\
\midrule
Student / teacher & Qwen2.5-7B-Instruct / sampled GPT-5-Chat responses \\
Warmup & one epoch of SFT on teacher responses (374 steps at $50\%$);
discriminator jointly initialized; GAD arms resume, SeqKD does not \\
Adversarial phase & two epochs; 748 steps at $50\%$, 492 at $33\%$ \\
Optimization & learning rate $10^{-6}$; batch 256; rollout group 8; temperature 0.8 \\
Lengths & maximum prompt 2048 tokens; maximum response 1536 \\
Regularization & KL coefficient $10^{-3}$; gradient clip 0.2; discriminator fp32 \\
Reference pool & 4096 BT rows; $\rho=0.5$; uniform sampling \\
SeqKD & learning rate $5\times10^{-6}$; batch 256; no discriminator or GRPO \\
Hardware & one node; $8\times$H200; tensor parallelism 1 \\
\bottomrule
\end{tabular}
\end{table}

A pool entry is a Bradley--Terry row $(x,y_T,y_S)$. The reference live-pool GAD
arm uses a FIFO queue tagged with the step that produced each row. With 256
discriminator rows per rank per step, capacity 4096 covers 16 steps. Each
pooled row replaces a fresh row; it does not enlarge the discriminator batch
or enter the policy update.

For exploratory live-pool sweeps, stale rows are sampled uniformly, by recency
$w_i=\exp[-\lambda(t-s_i)]$ with $\lambda=10^{-3}$, or by push-time loss
$w_i=p_i^\alpha$ with $p_i=\operatorname{softplus}(-m_i)$ and
$\alpha\in\{0.5,1\}$. Weighted draws use A-Res%
~\citep{efraimidis2006weighted}: sample $u_i\sim\mathrm{Uniform}(0,1)$, form
$u_i^{1/w_i}$, and take the largest keys. 

\section{Proofs of the Theoretical Results}
  \label{app:theory-proofs}

  We provide complete proofs of the three results in
  Section~\ref{sec:theory}. All expectations over learned score functions are
  taken over the randomness of the training procedure.

  \subsection{Proof of Theorem~\ref{thm:bayes-score}}

  \begin{proof}
  Fix a prompt \(x\) and suppress it from the notation. Let the finite response
  space be \(\mathcal Y=\{1,\ldots,m\}\), and write
  \(\tau_i=\tau(i\mid x)\), \(q_i=q(i\mid x)\), and \(r_i=r(x,i)\).
  The population loss is
  \[
    \mathcal L_D(r;q)
    =
    \sum_{i,j}\tau_iq_j
    \operatorname{softplus}(r_j-r_i).
  \]

  Terms with \(i=j\) are constant. For an unordered pair \(i\neq j\), let
  \(d=r_i-r_j\). Its contribution is
  \[
    g_{ij}(d)
    =
    \tau_iq_j\operatorname{softplus}(-d)
    +
    \tau_jq_i\operatorname{softplus}(d).
  \]

  Since all probabilities are positive, \(g_{ij}\) is strictly convex. Its
  derivative vanishes exactly when
  \[
    d
    =
    \log\frac{\tau_iq_j}{\tau_jq_i}
    =
    \log\frac{\tau_i}{q_i}
    -
    \log\frac{\tau_j}{q_j}.
  \]

  These pairwise conditions are simultaneously satisfied by
  \[
    r_i=\log\frac{\tau_i}{q_i}+c,
  \]
  where \(c\) is independent of the response. Because the complete loss is
  convex and this score minimizes every pairwise contribution, it is a global
  minimizer. Positivity of \(\tau_i\) and \(q_i\) makes the comparison graph
  connected, so all minimizing score differences are fixed. The only remaining
  freedom is the additive constant \(c\). Restoring the prompt dependence gives
  \[
    r_q^\star(x,y)
    =
    \log\frac{\tau(y\mid x)}{q(y\mid x)}
    +c_q(x).
  \]
  \end{proof}

  \subsection{Proof of Proposition~\ref{thm:pool-lower-bound}}

  \begin{proof}
  By linearity of expectation in the negative distribution,
  \[
    \mathcal L_t^\rho(r)
    =
    (1-\rho)\mathcal L_t^{\mathrm{fresh}}(r)
    +
    \rho\mathcal L_t^{\mathrm{hist}}(r).
  \]

  The softplus loss is nonnegative. Therefore,
  \[
    \mathcal L_t^\rho(r)
    \ge
    \rho\mathcal L_t^{\mathrm{hist}}(r).
  \]

  The softplus function is convex, so both component losses are convex in the
  score vector. Differentiating the mixture gives
  \[
    \nabla^2\mathcal L_t^\rho(r)
    =
    (1-\rho)\nabla^2\mathcal L_t^{\mathrm{fresh}}(r)
    +
    \rho\nabla^2\mathcal L_t^{\mathrm{hist}}(r).
  \]

  Under the curvature assumptions in the theorem, it follows that
  \[
    \nabla^2\mathcal L_t^\rho(r)
    \succeq
    \bigl[(1-\rho)\mu_{\mathrm{fresh}}
          +\rho\mu_{\mathrm{hist}}\bigr]I_{\mathcal S}
    \succeq
    \rho\mu_{\mathrm{hist}}I_{\mathcal S}.
  \]

  For completeness, consider the stated connectedness condition for the
  historical comparison graph. For any direction \(v\in\mathcal S\),
  \[
    v^\top\nabla^2\mathcal L_t^{\mathrm{hist}}(r)v
    =
    \sum_{i,j}
    \tau_i h_{t,j}\,
    \ell''(r_j-r_i)(v_j-v_i)^2,
  \]
  where \(\ell(u)=\operatorname{softplus}(u)\) and
  \(\ell''(u)=\sigma(u)\sigma(-u)>0\).

  On a bounded score set, \(\ell''(r_j-r_i)\) is uniformly bounded away from
  zero. If the historical comparison graph is connected, the quadratic form
  can vanish only when \(v_i=v_j\) for every connected pair. Thus \(v\) must be
  constant. The centered constraint \(v\in\mathcal S\) then implies \(v=0\).
  Hence the historical Hessian is uniformly positive definite on
  \(\mathcal S\), establishing \(\mu_{\mathrm{hist}}>0\) and completing the
  proof.
  \end{proof}

  \subsection{Proof of Proposition~\ref{thm:pool-bias-variance}}

  \begin{proof}
  Define the promptwise-centering operator
  \[
    (C_tf)(x,y)
    =
    f(x,y)
    -
    \mathbb E_{y'\sim\pi_{\theta_t}(\cdot\mid x)}f(x,y').
  \]
  Then \(\|f\|_{t,c}^2=\|C_tf\|_t^2\), where \(\|\cdot\|_t\) is the ordinary
  \(L_2\) norm under
  \(x\sim\mathcal D\) and \(y\sim\pi_{\theta_t}(\cdot\mid x)\).

  For either estimator \(\widehat r_t^a\), write
  \[
    \widehat r_t^a-r_t^\star
    =
    \bigl(\widehat r_t^a-\mathbb E\widehat r_t^a\bigr)
    +
    \bigl(\mathbb E\widehat r_t^a-r_t^\star\bigr).
  \]

  The first term has zero expectation. Consequently, the expected cross term
  vanishes after applying the linear centering operator \(C_t\), giving the
  bias--variance decomposition
  \[
    \mathbb E\|\widehat r_t^a-r_t^\star\|_{t,c}^2
    =
    V_a
    +
    \|\mathbb E\widehat r_t^a-r_t^\star\|_{t,c}^2.
  \]

  The assumed unbiasedness of the fresh estimator therefore gives
  \[
    \operatorname{MSE}(\widehat r_t^{\mathrm{fresh}},r_t^\star)
    =
    V_{\mathrm{fresh}},
  \]
  whereas the pooled estimator satisfies
  \[
    \operatorname{MSE}(\widehat r_t^{\mathrm{pool}},r_t^\star)
    =
    V_{\mathrm{pool}}+B_{\mathrm{pool}}^2.
  \]

  Using
  \(V_{\mathrm{pool}}\le V_{\mathrm{fresh}}-\Delta_t\), we obtain
  \[
    \operatorname{MSE}(\widehat r_t^{\mathrm{fresh}},r_t^\star)
    -
    \operatorname{MSE}(\widehat r_t^{\mathrm{pool}},r_t^\star)
    \ge
    \Delta_t-B_{\mathrm{pool}}^2.
  \]

  Finally, \(B_{\mathrm{pool}}^2<\Delta_t\) makes the right-hand side strictly
  positive. Hence the pooled estimator has lower reward MSE under the stated
  bias--variance condition.
  \end{proof}

\section{Seed-Level Robustness}
\label{app:seed-results}

We evaluate GAD and persistent-negative training using three paired random
seeds for each student family. Within each pair, both methods start from the
same seed-specific warmup checkpoint and use matched training configurations,
differing only in the discriminator negative distribution.

Figure~\ref{fig:judged-effects} reports the persistent-negative-minus-GAD
difference for each seed. Each effect is averaged equally across the four
judged-chat datasets. The individual points expose training-run variability;
the mean and standard deviation summarize the three paired effects. Because
only three seeds are available, the standard deviation is descriptive and
should not be interpreted as a precise uncertainty interval.

\begin{figure}[t]
\centering
\includegraphics[width=\linewidth]
{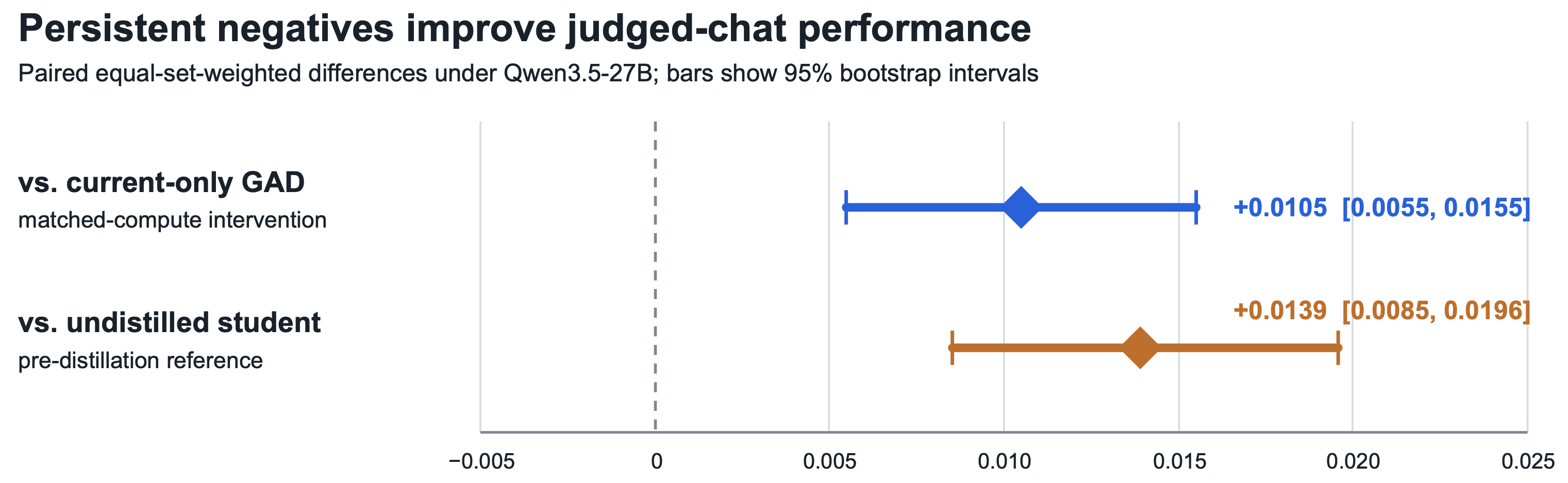}
\caption{\textbf{Performance gains across paired training seeds.}
Each point shows the equal-set-weighted improvement of persistent-negative
training over GAD for one paired seed, measured in percentage points.
Large markers denote the mean across three seeds, and error bars show one
standard deviation. Positive values favor persistent-negative training.}
\label{fig:judged-effects}
\end{figure}

\section{Additional Discriminator Diagnostics}
\label{app:discriminator-diagnostics}

\subsection{Fresh-policy discriminator accuracy}

We evaluate the discriminator immediately before each update using fresh
current-policy comparisons. The prompt--teacher slice is fixed across methods,
but each student generates its own negative responses. Consequently, this
diagnostic measures the behavior of the coupled student--discriminator system
rather than holding the student distribution fixed.

Table~\ref{tab:stability} reports results from the initial primary run across
748 adversarial updates. Persistent-negative training has slightly higher
mean accuracy, lower temporal variability, a higher minimum, and fewer
below-chance evaluations than GAD. Because consecutive measurements come from
the same training trajectory, these statistics are descriptive rather than
independent observations.

\begin{table}[t]
\centering
\small
\caption{\textbf{Fresh-policy discriminator accuracy in the initial primary
run.} Standard deviation is computed across adversarial updates, and
``dips'' counts pre-update evaluations below chance.}
\label{tab:stability}
\setlength{\tabcolsep}{4.5pt}
\renewcommand{\arraystretch}{0.96}
\begin{tabular}{@{}lcccc@{}}
\toprule
Method & Mean & Std. & Minimum & Dips $<0.5$ \\
\midrule
GAD
& 0.894 & 0.087 & 0.414 & 6 \\
Persistent-negative (ours)
& \textbf{0.903} & \textbf{0.068} & \textbf{0.461} & \textbf{1} \\
\bottomrule
\end{tabular}
\end{table}

These measurements are consistent with a smoother discriminator trajectory
under persistent-negative training. They do not, however, isolate changes in
the discriminator from changes in the evolving student policy.

\subsection{Common-probe reward geometry}

Fresh-policy accuracy does not reveal whether two discriminators assign
different rewards to identical responses. We therefore construct a common
probe by sampling eight responses for each of 128 held-out prompts from the
shared warmup policy. The resulting response groups are fixed and scored by
the final GAD and persistent-negative discriminators.

Across groups, the discriminators have a mean within-group Spearman rank
correlation of $0.803$ and a mean correlation of $0.846$ between their
standardized reward profiles. Thus, the learned reward functions are
substantially aligned, but they are not related solely by transformations
removed by promptwise GRPO normalization. Differences in response ordering
and relative score spacing therefore remain visible to the policy update.

This diagnostic uses one final discriminator pair and should be interpreted
as evidence that the pool changes the policy-facing reward geometry, not as
evidence that this change mediates the final performance improvement.

\end{document}